\documentclass[10pt,twocolumn,letterpaper]{article}

\usepackage{avss}
\usepackage{times}
\usepackage{epsfig}
\usepackage{graphicx}
\usepackage{amsmath}
\usepackage{amssymb}
\usepackage{url}
\usepackage{multirow}
\usepackage{algorithm}

\usepackage{algorithmic}

\usepackage{amsthm}
\newtheorem{lemma}{Lemma}

\usepackage[shortlabels]{enumitem}
\setlist[enumerate]{nosep}
\setlist[itemize]{nosep}

\usepackage{fancyhdr}

\usepackage[pagebackref=true,breaklinks=true,letterpaper=true,colorlinks,bookmarks=false]{hyperref}

\avssfinalcopy 

\def\avssPaperID{5}

\ifavssfinal\fi

\begin{document}

\title{Traffic Danger Recognition With Surveillance Cameras Without Training Data}

\author{Lijun Yu%
\thanks{This work was done when Lijun Yu was a visiting scholar at Carnegie Mellon University and later a research intern at MIX Labs.}\\
Peking University\\
Beijing, China\\
{\tt\small yulijun@pku.edu.cn}
\and
Dawei Zhang\\
MIX Labs\\
Beijing, China\\
{\tt\small dawei@mixlabs.xyz}
\and
Xiangqun Chen\\
Peking University\\
Beijing, China\\
{\tt\small cherry@sei.pku.edu.cn}
\and
Alexander Hauptmann\\
Carnegie Mellon Univ.\\ 
Pittsburgh, PA, US\\
{\tt\small alex@cs.cmu.edu}
}

\maketitle
\thispagestyle{fancy}
\fancyhead{} 
\lhead{} 
\lfoot{978-1-5386-9294-3/18/\$31.00 \copyright2018 IEEE}
\cfoot{} 
\rfoot{}

\begin{abstract}
    We propose a traffic danger recognition model that works with arbitrary traffic surveillance cameras to identify and predict car crashes. There are too many cameras to monitor manually. Therefore, we developed a model to predict and identify car crashes from surveillance cameras based on a 3D reconstruction of the road plane and prediction of trajectories. For normal traffic, it supports real-time proactive safety checks of speeds and distances between vehicles to provide insights about possible high-risk areas. We achieve good prediction and recognition of car crashes without using any labeled training data of crashes. Experiments on the BrnoCompSpeed dataset show that our model can accurately monitor the road, with mean errors of 1.80\% for distance measurement, 2.77 km/h for speed measurement, 0.24 m for car position prediction, and 2.53 km/h for speed prediction.
\end{abstract}

\section{Introduction}

Surveillance cameras are widely installed, recording and storing massive data every day. But anomalous events are very rare and it is impossible for humans to monitor all these cameras. Car crashes are a crucial safety issue nowadays. Leveraging the recent development of computer vision algorithms, we are developing an automatic system for traffic surveillance on highways and streets.

We have built a model that can predict and recognize crashes from surveillance cameras. One benefit is that ambulances could immediately be sent to the crash scene saving lives. As accidents are relatively few, our model also supports proactive safety check based on normal traffic flow. Real-time speed and distance measurements will lead to insights about high-risk areas, such as where cars frequently get too close. This will help to improve traffic safety on the long term.

As accidents are rare in regular surveillance videos, it is arduous to collect and build a labeled dataset of car crashes covering all possible situations. Taking this reality into account, we propose a model that requires no labeled crash data for training. Physically, a collision between cars occurs when they gradually get closer and finally come into contact. We can predict their trajectories and check overlap positions indicating a collision. In severe crashes, vehicles are deformed and undetectable afterwards, but the crash is recognized ahead of time based on the predictions.

Our model consists of five steps to achieve the goal. Camera calibration method is applied to transform a point on the image to the road plane. Object detection and tracking algorithms identify a vehicle and trace its history. A 3D bounding box is built to get the projection of a car on the road. Position and speed are estimated and predicted for the future. Finally, the model can recognize danger based on distances between vehicles and overlaps in the trajectories.

We run this model on BrnoCompDataset~\cite{sochor2018comprehensive}, which contains highway surveillance videos with ground truth speed and distance measurements. We evaluate its performance for different steps and show convincing results.  It performs an effective 3D reconstruction of the road plane with a mean distance measurement error of 1.80\% along the road. Upon efficient detection and tracking of vehicles, it takes a precise measurement of speeds at a mean error of 2.77 km/h. It predicts vehicle trajectories reliably with errors of 0.24 m for car positions and 2.53 km/h for speeds averagely for 0.12 seconds ahead. This allows refined recognition of traffic danger from all the measurement and predictions. Importantly, all these results are achieved without any labeled training data.

The key contributions of this paper are:

\begin{itemize}
    \item A traffic danger recognition model for surveillance cameras based on a 3D reconstruction of the road and prediction of trajectories. It does not need any labeled training data of car crashes.
    \item Results show that the model monitors the road accurately, with mean errors of 1.80\% for distance measurement, 2.77 km/h for speed measurement, 0.24 m for car position prediction, and 2.53 km/h for speed prediction.
\end{itemize}

\section{Related Work}

\textbf{Camera Calibration.} Calibration methods are employed to derive the intrinsic (focal length, principal point) and extrinsic (rotation, translation) parameters of a camera. The accuracy of calibration is critical for the 3D reconstruction and further processing. Different methods may require various forms of user inputs, such as drawing parallel lines~\cite{lee2011robust}, camera position~\cite{wang2007research, pai2001adaptive}, and average vehicle size~\cite{dailey2000algorithm} or speed~\cite{schoepflin2003dynamic}. Fully automatic calibration is also achievable according to~\cite{dubska2015fully, sochor2017traffic}.

\textbf{Object Detection.} Object detection models are utilized to identify vehicles in video frames. These models such as Fast R-CNN~\cite{girshick2015fast} and Faster R-CNN~\cite{ren2015faster} rely on region proposal algorithms and deep convolution neural networks to get bounding boxes of objects. Mask R-CNN~\cite{he2017mask} further extends by predicting object masks simultaneously.

\textbf{Multiple Object Tracking.} Vehicle objects detected in adjacent frames need to be traced correctly. SORT algorithm~\cite{bewley2016simple} supports fast online tracking with Kalman Filter~\cite{kalman1960new} and Hungarian algorithm~\cite{kuhn1955hungarian}. Deep-SORT~\cite{wojke2017simple} additionally integrates appearance information to improve the performance.

\textbf{Anomaly Detection.} Traffic danger recognition is one specific aspect of anomaly detection. Multiple instance learning~\cite{sultani2018real} requires sufficient annotated training data. Motion pattern based learning for traffic anomaly~\cite{xu2018dual} also uses labeled data. But our approach is built upon no labeled videos of car crashes.

\section{Methodology}

Our traffic danger recognition model consists of five steps. Camera calibration provides geometry parameters and a transformation from image coordinates to road plane coordinates. Object detection and tracking algorithms provide the types, positions, and masks of vehicles and trace their histories. 3D bounding boxes are built to localize vehicles in the world space and then project to the road plane. Positions and speeds are calculated with adjacent frames plus smoothing and predicted for the future. Finally, we can recognize danger from vehicle distances and potential overlaps in the predictions.

\subsection{Camera Model and Calibration}
\label{sec:camera_calibration}

We adopt a traffic camera model similar to the paper of Sochor \etal~\cite{sochor2017traffic} as shown in Figure~\ref{fig:camera_model}. We follow the practice of Dubsk{\'a} \etal~\cite{dubska2014automatic} in setting up directions of three vanishing points $U, V, W$. With a known plane, points in an image can be reprojected to points on the plane in the world space. The reprojection enables a 3D reconstruction of vehicles on the road.

\begin{figure}[h]
    \setlength{\abovecaptionskip}{0.cm}
    \setlength{\belowcaptionskip}{-0.cm}
    \begin{center}
        \includegraphics[width=0.6\linewidth]{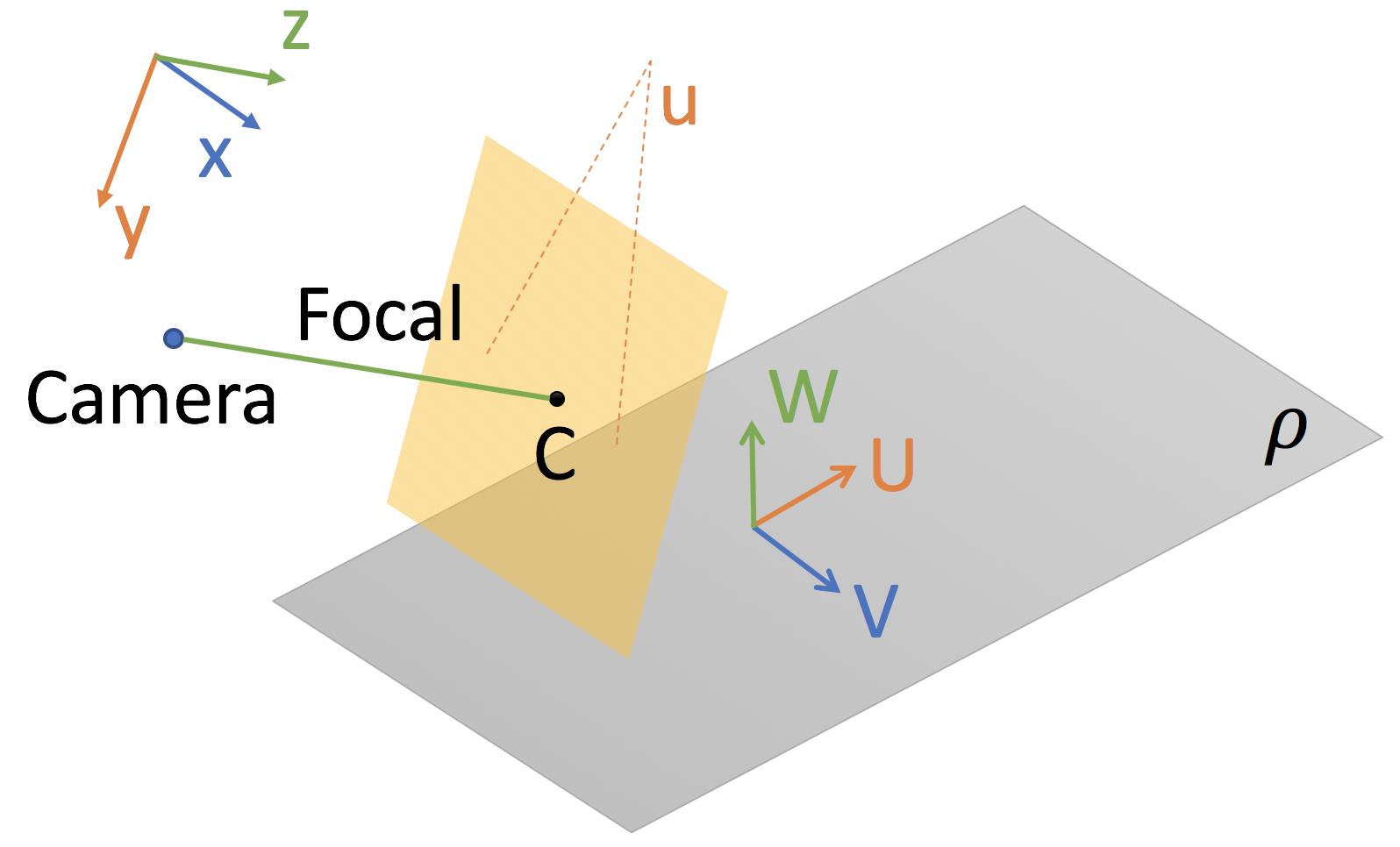}
    \end{center}
    \caption{Traffic camera model: $x, y, z$ define world coordinate system, where $x$-$y$ plane is parallel to the image and $z$ passes through its top left. Camera is on the $x$-$y$ plane and points to the principal point $C$ at the center of the image. $\rho$ is the road plane. $U, V, W$ are directions of vanishing points, $U$ in the direction of traffic, $V$ parallel to the road and perpendicular to $U$, and $W$ perpendicular to $\rho$.}
    \label{fig:camera_model}
\end{figure}

Although some automatic calibration methods have been developed, they do not achieve perfect performance in our model. So we remain using a manual calibration which requires labeling two groups of parallel lines of each camera view. Then we derive two vanishing points in the image space using a least square error method as in~\cite{lee2011robust}. With Algorithm~\ref{alg:vp_plane} extracted from the supplementary material of the dataset~\cite{sochor2018comprehensive}, we can derive the road plane in the world space and project image points to world points on the plane.

We rotate the world coordinate system to make the $x$-$z$ plane parallel to the road plane, so we can get plane coordinates of a point by omitting $y$ axis. Rotation parameters $\alpha, \beta, \gamma$ are acquired by solving Equation~\ref{eq:rotate}.

\begin{equation}
    \setlength{\belowdisplayskip}{2pt}
    \setlength{\abovedisplayskip}{2pt}
    \label{eq:rotate}
    \begin{aligned}
    \left[
        \begin{smallmatrix}
            1 & 0 & 0 & 1 \\
            0 & 1 & 0 & 1 \\
            0 & 0 & 1 & 1
        \end{smallmatrix}
    \right] \cdot \left[
        \begin{smallmatrix}
            1 & 0 & 0 & 0 \\
            0 & \cos{\alpha} & \sin{\alpha} & 0 \\
            0 & -\sin{\alpha} & \cos{\alpha} & 0 \\
            0 & 0 & 0 & 1
        \end{smallmatrix}
    \right] \cdot \left[
        \begin{smallmatrix}
            \cos{\beta} & 0 & -\sin{\beta} & 0 \\
            0 & 1 & 0 & 0 \\
            \sin{\beta} & 0 & \cos{\beta} & 0 \\
            0 & 0 & 0 & 1
        \end{smallmatrix}
    \right] \\
    \cdot \left[
        \begin{smallmatrix}
            \cos{\gamma} & \sin{\gamma} & 0 & 0 \\
            -\sin{\gamma} & \cos{\gamma} & 0 & 0 \\
            0 & 0 & 1 & 0 \\
            0 & 0 & 0 & 1
        \end{smallmatrix}
    \right] = \left[
        \begin{smallmatrix}
            \frac{V}{\left\|V\right\|} \\
            \frac{W}{\left\|W\right\|} \\
            \frac{U}{\left\|U\right\|}
        \end{smallmatrix}
    \right]
    \end{aligned}
\end{equation}

\begin{algorithm}
    \caption{Project an image point to a world point on the road plane. $f$ denotes the focal length. For points, lower case represents image coordinates and upper case stands for world coordinates. $d$ is an arbitrary offset of the plane, usually set to 10 as in~\cite{sochor2018comprehensive}.}
    \label{alg:vp_plane}
    \begin{algorithmic}
        \REQUIRE Image point $p = [p_x, p_y]$, plane offset $d$
        \REQUIRE Calibration $u = [u_x, u_y], v = [v_x, v_y], c = [c_x, c_y]$
        \ENSURE World point on the plane $P = [P_x, P_y, P_z]$
        \STATE $f = \sqrt{-(u - c) \cdot (v - c)}$
        \STATE $U = [u_x, u_y, f], V = [v_x, v_y, f], C = [c_x, c_y, 0]$
        \STATE $W = [W_x, W_y, W_z] = (U - C) \times (V - C)$
        \STATE $w = [w_x, w_y] = \frac{[W_x, W_y]}{W_z} \cdot f + c$
        \STATE $n = [n_x, n_y, n_z] = [w_x, w_y, f] - C$
        \STATE $\rho = [a, b, c, d] = [\frac{n_x}{\left\|n\right\|}, \frac{n_y}{\left\|n\right\|}, \frac{n_z}{\left\|n\right\|}, d]$
        \STATE $g = [p_x, p_y, f] - C$
        \STATE $t = -\frac{\rho \cdot [c_x, c_y, 0, 1]}{[a, b, c] \cdot g}$
        \STATE $P = C + t \cdot g$
    \end{algorithmic}
\end{algorithm}

\subsection{Object Detection and Tracking}

We select Mask R-CNN by He \etal~\cite{he2017mask} as our object detection model, which outputs detection scores, object types, bounding boxes, and object masks. We use Abdulla's implementation~\cite{matterport_maskrcnn_2017} with trained weights on Microsoft COCO dataset~\cite{lin2014microsoft} and select three types of objects as targets: car, bus, and truck. Then we apply a filter to the detected objects as shown in Figure~\ref{fig:detection}. The filter follows three rules:

\begin{enumerate}
    \item Vehicles should not be too small in size.
    \item Vehicles should be in the road area.
    \item Vehicles should be completely visible.
\end{enumerate}

\begin{figure}[htbp]
    \setlength{\abovecaptionskip}{0.cm}
    \setlength{\belowcaptionskip}{-0.cm}
    \begin{center}
        \includegraphics[width=\linewidth]{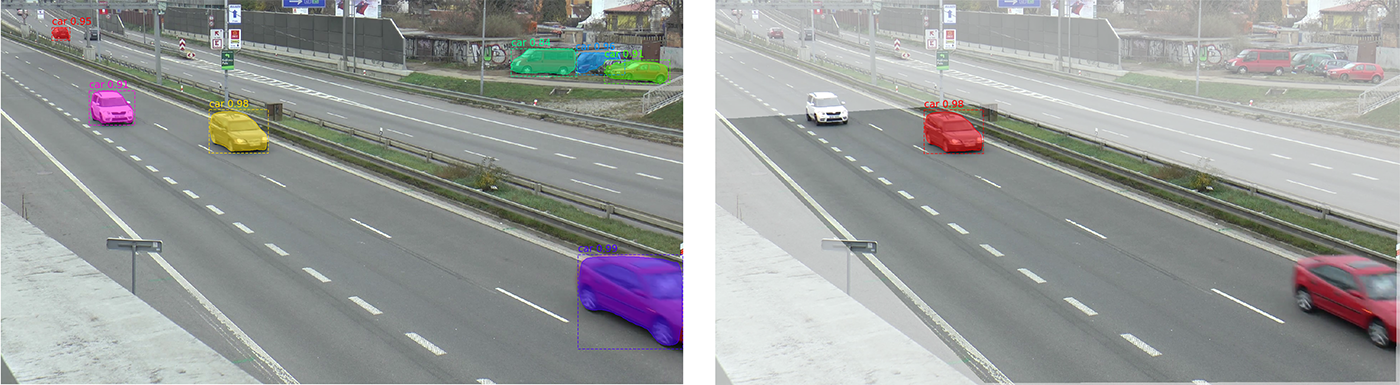}
    \end{center}
    \caption{Object detection: raw detections (left), and filtered objects (right). The white car at the top left is filtered by rule 1, the red at the bottom right by 3, and the cars at the top right by 2.}
    \label{fig:detection}
\end{figure}

We use Deep SORT by Wojke \etal~\cite{wojke2017simple} to track vehicles across frames. Each vehicle is supposed to get a unique ID from the tracking model, and it is robust through brief loss of detection.

\subsection{3D Bounding Box}

We estimate the contour of a vehicle with its mask from Mask R-CNN, using the algorithm by Suzuki \etal~\cite{suzuki1985topological}. For each of the three vanishing points, we calculate the tilt angles of the lines passing that vanishing point and each point in the contour. In this way we find the tangent lines of the contour passing three vanishing lines. We alter the algorithm from Sochor~\cite{sochor2014traffic} to build 3D bounding boxes of cars as described in Algorithm~\ref{alg:3d_box} and Figure~\ref{fig:3d_bounding_box}.

\begin{algorithm}
    \caption{Build 3D bounding box with tangent lines from vanishing points. Lines with subscript $min$ denote the lines with the minimum tilt angle, and $max$ the maximum. Position of the points are shown in Figure~\ref{fig:3d_bounding_box}.}
    \label{alg:3d_box}
    \begin{algorithmic}
        \REQUIRE Tangent lines from vanishing points $l_{U, min}, l_{U, max},$ \\
        $ l_{V, min}, l_{V, max}, l_{W, min}, l_{W, max}$
        \ENSURE 3D bounding box $(A, B, C, D, E, F, G, H)$
        \STATE $A = l_{U, max} \cap l_{V, min}$
        \STATE $B = l_{V, min} \cap l_{W, max}$
        \STATE $D = l_{U, max} \cap l_{W, min}$
        \STATE $F = l_{U, min} \cap l_{W, max}$
        \STATE $G = l_{V, max} \cap l_{U, min}$
        \STATE $H = l_{V, max} \cap l_{W, min}$
        \STATE $E_F = \overline{FV} \cap \overline{AW}$
        \STATE $E_H = \overline{HU} \cap \overline{AW}$
        \IF{$\left| AE_F \right| \ge \left| AE_H \right|$}
            \STATE $E = E_F$
        \ELSE
            \STATE $E = E_H$
        \ENDIF
        \STATE $H = \overline{CV} \cap \overline{BU}$
    \end{algorithmic}
\end{algorithm}

\begin{figure}[h]
    \setlength{\abovecaptionskip}{0.cm}
    \setlength{\belowcaptionskip}{-0.cm}
    \begin{center}
        \includegraphics[width=\linewidth]{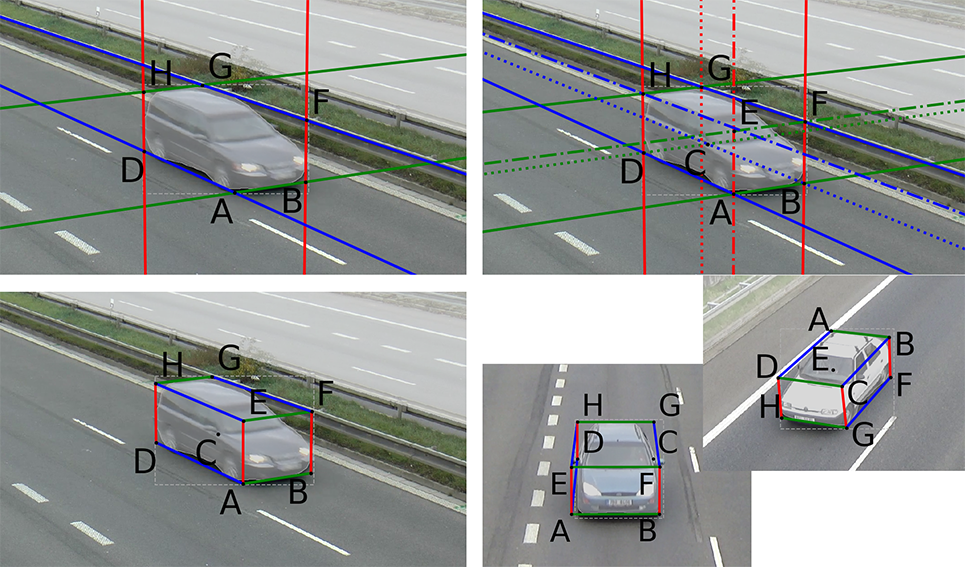}
    \end{center}
    \caption{3D bounding box: tangent lines of contour and their intersections (top left), derived lines and intersections (top right), the final result (bottom left), and vehicles in other angles of view (bottom right). Lines in colors of blue, green, and red pass through $u, v, w$ respectively.}
    \label{fig:3d_bounding_box}
\end{figure}

\subsection{Trajectory Prediction}

To get the current location of a vehicle, we can find the bottom of the 3D bounding boxes and project them to the road plane according to Section~\ref{sec:camera_calibration}. The set of the bottom points $S$ relies on the direction of the vehicle as:
\begin{equation}
    \setlength{\belowdisplayskip}{2pt}
    \setlength{\abovedisplayskip}{2pt}
    S =
    \begin{cases}
        \{A, B, C, D\} \mbox{, when} \tan{\overrightarrow{DA}} \ge 0\\
        \{H, G, F, E\} \mbox{, otherwise}
    \end{cases}
\end{equation}
The center position of a vehicle is calculated by 
\begin{equation}
    \setlength{\belowdisplayskip}{2pt}
    \setlength{\abovedisplayskip}{2pt}
    c = average(S)
\end{equation}
and a recent speed $v_r$ is calculated from adjacent frames as
\begin{equation}
    \setlength{\belowdisplayskip}{2pt}
    \setlength{\abovedisplayskip}{2pt}
    v_{r, t} = (c_t - c_{t-1}) \times \mathit{fps}
\end{equation}
where $t$ denotes frame number and $\mathit{fps}$ is the frame rate of the video. Exponential smoothing is applied to get a smoothed speed $v_s$ as
\begin{equation}
    \setlength{\belowdisplayskip}{2pt}
    \setlength{\abovedisplayskip}{2pt}
    v_{s, t} = \delta \times v_{s, t-1} + (1 - \delta) \times v_{r, t}
\end{equation}
With an optional scale factor $\lambda$, we are able to know the real world value of the speed.

To predict the trajectories, we assume:
\begin{enumerate}
    \item The future is divided into time slots with equal lengths.
    \item The vehicle centers follow normal distributions.
    \item The vehicle shapes do not change.
\end{enumerate}
We predict speed, acceleration, center coordinates and variance for the beginning of each slot as a snapshot. Within a slot, we assume there are fixed acceleration and variance. Then the speed and center coordinates can be calculated according to kinematics rules. In this way, predictions are available for an arbitrary time in the future.

For now, we are using a simple linear prediction method with the real situation as the only one snapshot and assuming the acceleration is always zero. Models of conditional random fields~\cite{lafferty2001conditional} and long short-term memory~\cite{hochreiter1997long} are planned to be tested in the future.

\subsection{Danger Recognition}

We use two ways to recognize dangerous situations. The first one is the distance measurement between vehicles. It not only tells where cars are going to crash but provides a proactive safety check for areas where cars often get too close, as well. The second one is called danger map, which detects overlap of vehicles in the predictions that indicates crashes.

The distance between two vehicles is defined as the minimum distance between two points from two quadrangles respectively.

\begin{lemma}
    Let $A, B$ be the pair with the minimum distance among all pairs of points from two quadrangles respectively, then at least one of $A, B$ must be a vertex.
    \label{lem:distance}
\end{lemma}

\begin{proof}
    Suppose both $A, B$ are not vertices, so each of them is on an edge, namely $a, b$. If $a \parallel b$, there must be another pair of points consisting of at least one vertex that has an equal distance. If $a$ is not parallel to $b$, then the nearest distance between $a$ and $b$ cannot be at the middle of both edges, which contradicts the suppose. Therefore, at least one of $A, B$ is a vertex.
\end{proof}

With Lemma~\ref{lem:distance}, we can calculate the minimum distance between two quadrangles $A_1B_1C_1D_1, A_2B_2C_2D_2$ as:

\begin{equation}
    \setlength{\belowdisplayskip}{2pt}
    \setlength{\abovedisplayskip}{2pt}
    \begin{aligned}
        d_{qq}(A_1B_1&C_1D_1, A_2B_2C_2D_2) = \\
        &\begin{aligned}
            \min( & \min_{P=A_1B_1C_1D_1}{d_{pq}(P, A_2B_2C_2D_2)}, \\
            & \min_{P=A_2B_2C_2D_2}{d_{pq}(P, A_1B_1C_1D_1)})
        \end{aligned}
    \end{aligned}
\end{equation}
\begin{equation}
    \setlength{\belowdisplayskip}{2pt}
    \setlength{\abovedisplayskip}{2pt}
    d_{pq}(P, ABCD) = \min_{e=AB, BC, CD, DA}d_{pe}(X, e)
\end{equation}
where $d_{pe}$ is the distance between a point and an edge. In this way, the minimum distance is calculated from only 32 candidates. Distances for all vehicle pairs are calculated and alerted when less than a threshold.

We accumulate the probability of a car box based on the distribution of its center to get the heat map of a vehicle. It represents the probability of its position at a specific time in the future. Then we aggregate the heat maps of all the vehicles in a scene into a danger map. A danger map represents the probability of coexistence of two or more vehicles in the same location. Figure~\ref{fig:danger_recognition} shows a sample result of danger recognition.

\begin{figure}[h]
    \setlength{\abovecaptionskip}{0.cm}
    \setlength{\belowcaptionskip}{-0.cm}
    \begin{center}
        \includegraphics[width=\linewidth]{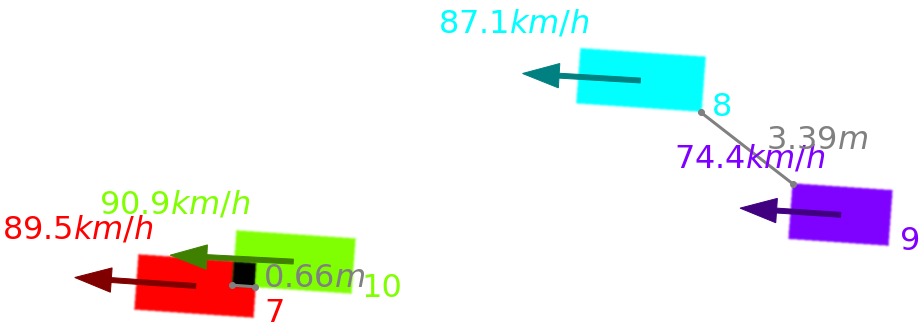}
    \end{center}
    \caption{Danger recognition: four vehicles in a sample prediction of the road plane, with vehicle IDs at the bottom right and speeds at the top left. Distances are shown for nearing vehicles, and danger map is shown in black at the overlap of vehicle 7 and 10.}
    \label{fig:danger_recognition}
\end{figure}

\section{Experiments}

\subsection{Dataset and Setup}

We use BrnoCompDataset~\cite{sochor2018comprehensive} to evaluate the performance of our model. It consists of surveillance videos of 6 sessions from 3 directions on the highway in the Czech Republic. The dataset provides the ground truth of distance measurement lines and speed of vehicles from Lidar sensor. It also has calibration results from various systems~\cite{sochor2018comprehensive, sochor2017traffic}.

We run our model on each of the 18 videos for 10 minutes. The videos are processed at the original resolution of 1080p and downsampled from 50 fps to 25 fps. We do not use lower frame rate because the Deep-SORT model has worse performance when it is less than 25fps. We exploit the calibration results from~\cite{sochor2017traffic} which provides vanishing points $u, v$ acquired by manual calibration from parallel lines, along with scale factors inferred from speeds. We let the smoothing parameter $\delta = 0.86$ according to some preliminary experiments. The trajectory prediction is set for 0.12 and 0.24 seconds ahead, accordingly 3 and 6 frames.

\subsection{Calibration Error}

We measure the calibration error to test the provided calibration results and the correctness of our coordinates transformation algorithm which maps a point from the image space to the road plane space.

We calculate the distance of the given measurement lines in our plane coordinate system. The lines are divided into two groups according to their directions: toward $u$ or $v$. The average length of the given lines is different in each group. We collect absolute and relative errors of measured distances and report the mean and median values in Table~\ref{tab:calibration}. 

\begin{table}[h]
    \setlength{\abovecaptionskip}{0.cm}
    \setlength{\belowcaptionskip}{-0.cm}
    \begin{center}
        \begin{tabular}{|c|c|c|c|}
            \hline
             &  & Mean & Median \\ \hline
            \multirow{2}{*}{\begin{tabular}[c]{@{}c@{}}Distance\\ Toward $u$\end{tabular}} & Absolute Error (m) & 0.2618 & 0.1684 \\ \cline{2-4} 
             & Relative Error & 1.80\% & 1.42\% \\ \hline
            \multirow{2}{*}{\begin{tabular}[c]{@{}c@{}}Distance\\ Toward $v$\end{tabular}} & Absolute Error (m) & 0.1633 & 0.1646 \\ \cline{2-4} 
             & Relative Error & 2.06\% & 2.07\% \\ \hline
        \end{tabular}
    \end{center}
    \caption{Calibration error: distance measurement error on all videos}
    \label{tab:calibration}
\end{table}

The results show that our model can accurately measure distances in the real world based merely on surveillance camera views and calibration parameters. The error in each direction is much smaller than the shape of conventional vehicles. As of the high speed in the highway, these errors are even smaller than the movement of a vehicle between two adjacent frames. This model provides an effective 3D reconstruction of the road plane with little error.

\subsection{Vehicle Detection and Tracking Error}

We measure vehicle detection and tracking error to test the Mask R-CNN and Deep-SORT models. For each vehicle detected and tracked, we record the time and position of its every appearance. Based on the appearance history, we calculate an estimated period of the vehicle in the measurement area of Lidar sensors. The measurement area is considered to be the largest one if there are more than two Lidar sensors set up. Then we calculate the intersection of union (IoU) between the estimated period and the real period of existence in the ground truth to get a similarity matrix. Hungarian algorithm~\cite{kuhn1955hungarian} is employed to solve this matching problem. Additionally, matching results with IoU less than $l_{IoU}=0.5$ are dropped. We report the recall on each video for this evaluation in Table~\ref{tab:detection_tracking}.

We find that Mask R-CNN sometimes does not work at certain viewing angles or for certain types of vehicles. For lost detections, as long as the gap is short enough, Deep-SORT is still able to track. In other cases, however, tracking also fails and that causes the loss. Despite these, the combination of Mask R-CNN and Deep-SORT have achieved an overall $94.0\%$ recall rate, which shows that it is efficacious for the vehicle detection and tracking in this task.

\subsection{Speed Estimation Error}

We use the matched vehicles from the previous section to evaluate the performance of our speed estimation. As the ground truth only has the average speed for each vehicle, we use the smoothed speed of a vehicle at its last appearance for comparison. We collect absolute and relative errors of the estimated speed of each vehicle and report in Table~\ref{tab:speed}.

\begin{table}[h]
    \setlength{\abovecaptionskip}{0.cm}
    \setlength{\belowcaptionskip}{-0.cm}
    \begin{center}
        \begin{tabular}{|c|c|c|}
            \hline
             & Mean & Median \\ \hline
            Absolute Error (km/h) & 2.7708 & 1.8625 \\ \hline
            Relative Error & 3.68\% & 2.55\% \\ \hline
        \end{tabular}
    \end{center}
    \caption{Speed estimation error: estimated speeds compared with ground truth from Lidar sensors on all videos}
    \label{tab:speed}
\end{table}

According to the dataset, the average speed for each session is mostly between 60 km/h and 90 km/h. For highway traffic, a mean error of less than 2.77 km/h proves that our model can precisely measure the speeds of vehicles. This accurate measurement is the foundation for further predictions and danger recognition.

\subsection{Prediction Error}

We evaluate the two levels of predictions separately. For each level, we collect the absolute error of location prediction, plus the absolute and relative error of speed prediction of each vehicle. As the smoothed speed is not stable at the beginning, predictions from vehicles with a history of fewer than $l=5$ frames (0.2 seconds) are excluded. The mean and median values of each metric are shown in Table~\ref{tab:prediction}.

\begin{table}[h]
    \setlength{\abovecaptionskip}{0.cm}
    \setlength{\belowcaptionskip}{-0.cm}
    \begin{center}
        \begin{tabular}{|c|c|c|c|}
            \hline
            Level & +0.12s & Mean & Median \\ \hline
            \begin{tabular}[c]{@{}c@{}}Location\\ Prediction\end{tabular} & Absolute Error (m) & 0.2433 & 0.1736 \\ \hline
            \multirow{2}{*}{\begin{tabular}[c]{@{}c@{}}Speed\\ Prediction\end{tabular}} & Absolute Error (km/h) & 2.5313 & 1.8373 \\ \cline{2-4} 
             & Relative Error & 4.55\% & 2.52\% \\ \hline
        \end{tabular}
        \begin{tabular}{|c|c|c|c|}
            \hline
            Level & +0.24s & Mean & Median \\ \hline
            \begin{tabular}[c]{@{}c@{}}Location\\ Prediction\end{tabular} & Absolute Error (m) & 0.3563 & 0.3256 \\ \hline
            \multirow{2}{*}{\begin{tabular}[c]{@{}c@{}}Speed\\ Prediction\end{tabular}} & Absolute Error (km/h) & 3.0134 & 2.4995 \\ \cline{2-4} 
             & Relative Error & 5.71\% & 3.92\% \\ \hline
        \end{tabular}
    \end{center}
    \caption{Prediction error: location and speed prediction error of two levels on all videos}
    \label{tab:prediction}
\end{table}

Although the prediction mechanism currently deployed is rather simple, it provides results much beyond our expectations. As a vehicle at 75 km/h would move 2.5 meters in 0.12 seconds, a mean error of 0.24 m for location prediction is well acceptable. The difference between the mean and median values indicates some outliers are harming the performance, but we can still see that most of the predictions are within an error of 2km/h. For traffic on highways, crashes usually happen within 0.12 seconds, so it is enough for the danger map to work. Moreover, another prediction of +0.24s is there for more information beforehand, and it is reasonable to have a slightly larger error than +0.12s.

\begin{table*}[htbp]
    \begin{center}
        \begin{tabular}{|c|c|c|c|c|c|c|c|c|c|c|}
            \hline
            Video ID & 1C & 1L & 1R & 2C & 2L & 2R & 3C & 3L & 3R & 4C \\ \hline
            Vehicle Matching Recall & 95.0\% & 92.2\% & 97.0\% & 82.2\% & 92.6\% & 92.5\% & 81.8\% & 100\% & 100\% & 92.4\% \\ \hline
            Video ID & 4L & 4R & 5C & 5L & 5R & 6C & 6L & 6R &  & Mean \\ \hline
            Vehicle Matching Recall & 93.2\% & 98.2\% & 83.0\% & 98.9\% & 97.5\% & 98.8\% & 99.5\% & 96.4\% &  & 94.0\% \\ \hline
            \end{tabular}
    \end{center}
    \caption{Vehicle detection and tracking error: vehicle matching recall on each video. The number in Video ID is Session ID, and the letter denotes the direction according to C-center, L-left, R-right.}
    \label{tab:detection_tracking}
\end{table*}

\section{Conclusions}

We propose a traffic danger recognition model that works with arbitrary surveillance cameras. It does not require any labeled training data of crashes. The model consists of five steps: camera calibration, object detection and tracking, 3D bounding box, trajectory prediction, and danger recognition. We measure the performance with experiments step by step, presenting that it is accurate at the estimation of speed and position of vehicles by projecting to a 3D reconstructed road plane. It is suitable for crash detection and proactive safety checks.

A demo of our model working on a real crash scene can be found on Youtube\footnote{\url{https://www.youtube.com/playlist?list=PLssAerj8zfUR5wBc7N6gmCFTm0azCHSIf}}. In the future, a complete test set of video containing real crashes will be processed to report detection accuracy. Trajectory prediction model could be improved with conditional random fields or recurrent neural network. We will also test automatic camera calibration methods to obtain similar performance as manual calibration, then the system could function on arbitrary surveillance cameras with zero input.

{\small
\bibliographystyle{ieee}
\bibliography{reference}
}

\end{document}